%% file: icml2024.tex
\documentclass{article}

\usepackage{microtype}
\usepackage{graphicx}
\usepackage{subfigure}
\usepackage{booktabs} 

\usepackage{hyperref}

\usepackage[accepted]{icml2024}

\usepackage{amsmath}
\usepackage{amssymb}
\usepackage{adjustbox}
\usepackage{mathtools}
\usepackage{amsthm}
\usepackage[english]{babel}
\usepackage{booktabs} 
\usepackage{pgfplots}
\usepackage{paralist}
\pgfplotsset{compat=newest}
\usepackage{tikz} 
\usepackage{listings}

\usepackage{cleveref}

\usepackage[textsize=tiny]{todonotes}

\input{new_commands_and_packages}

\newcommand{\runningtitle}{Bayesian Reward-conditioned Amortized I\textsc{n}ference}
\newcommand{\maintitle}{\shortname: Bayesian Reward-conditioned Amortized \textsc{In}ference \\ for natural language generation from feedback}

\icmltitlerunning{\runningtitle}

\begin{document}

\twocolumn[
\icmltitle{\maintitle}



\icmlsetsymbol{equal}{*}

\begin{icmlauthorlist}
\icmlauthor{Gaurav Pandey}{comp1}
\icmlauthor{Yatin Nandwani}{comp1}
\icmlauthor{Tahira Naseem}{comp1}
\icmlauthor{Mayank Mishra}{comp1}
\icmlauthor{Guangxuan Xu}{comp1}
\icmlauthor{Dinesh Raghu}{comp1}
\icmlauthor{Sachindra Joshi}{comp1}
\icmlauthor{Asim Munawar}{comp1}
\icmlauthor{Ramón Fernandez Astudillo}{comp1}
\end{icmlauthorlist}

\icmlaffiliation{comp1}{IBM Research AI}


\icmlcorrespondingauthor{Gaurav Pandey}{gpandey1@in.ibm.com}

\icmlkeywords{Machine Learning, ICML}

\vskip 0.3in
]



\printAffiliationsAndNotice{}  


\begin{abstract}
\input{icml2024/sections/abstract}
\end{abstract}

\input{icml2024/sections/introduction}

\input{icml2024/sections/relatedworks}

\input{icml2024/sections/technique}
\input{icml2024/sections/experiments}

\input{icml2024/sections/conclusions}
\input{icml2024/sections/impact}


\bibliography{icml2024/icml2024}
\bibliographystyle{icml2024}

\newpage
\appendix
\onecolumn
\input{icml2024/sections/appendix}


\end{document}

%% file: new_commands_and_packages.tex
\usepackage{bbm}
\usepackage{booktabs} 
\usepackage{wrapfig}
\usepackage{comment}
\usepackage{cleveref}
\usepackage[nounderscore]{syntax}

\usepackage{thmtools,thm-restate}

\theoremstyle{plain}

\theoremstyle{definition}

\theoremstyle{remark}

\newcommand{\ie}{\mbox{\it i.e.}}

\newcommand{\wrt}{\mbox{\it w.r.t.}}

\newcommand{\viz}{\mbox{\it viz.}}

\newcommand{\shortname}{\textsc{BRAIn}}
\newcommand{\longname}{\textbf{B}ayesian \textbf{R}eward-conditioned \textbf{A}mortized \textbf{In}ference}

\newcommand{\bigprob}{\mathbb{P}}
\newcommand{\yset}{\mathcal{Y}}

\newcommand{\KL}{D_{KL} }
\newcommand{\KLnorm}{D_{n}^{q'}}
\newcommand{\inp}{x}
\newcommand{\outp}{y}
\newcommand{\outpp}{{\outp}}
\newcommand{\outpvec}{\mathbf{y}}
\newcommand{\outppvec}{{\mathbf{y}}}

\newcommand{\rewardrv}{G}
\newcommand{\goodness}{g}
\newcommand{\ppolicy}{q_{\theta}}
\newcommand{\pposterior}{p}
\newcommand{\pproposal}{q'}
\newcommand{\entropy}{\mathcal{H}}

\newcommand{\dposft}{DPO-sft}

\newcommand{\alphahatyx}{\hat{\alpha}_{Y_x}}
\newcommand{\betahatthetayx}{\hat{\beta}_{Y_x}}

\newcommand{\alphay}[1]{\alpha_{\outp_{#1}}}
\newcommand{\alphahaty}[1]{\hat{\alpha}_{\outp_{#1}}}
\newcommand{\betathetay}[1]{\beta_{\outp_{#1}}}
\newcommand{\betahatthetay}[1]{\hat{\beta}_{\outp_{#1}}}

\newcommand{\rewardfn}[1]{r(\inp,{#1})}
\newcommand{\prior}[1]{p({#1}|\inp)}
\newcommand{\posterior}[1]{p({#1} | \inp, \rewardrv=1)}
\newcommand{\policy}[1]{q_{\theta}({#1} | \inp)}

\newcommand{\proposal}[1]{q'({#1} | \inp)}

\newcommand{\optimal}[1]{p^{*}({#1} | \inp)}

\newcommand{\rewardprob}[1]{p(\rewardrv=1|{#1},\inp)}
\newcommand{\rewardprobx}{p(\rewardrv=1| \inp)}

\newcommand{\expectation}{\mathbb{E}}
\newcommand{\loss}{\mathcal{L}_x(\theta)}
\newcommand{\lossestimate}{\hat{\mathcal{L}}_x(\theta)}
\newcommand{\multiexpectation}{\mathbb{E}_{\substack{\outp_i \sim \proposal{\outp} \\ 1 \leq i \leq n}}}
\newcommand{\multiexpectationdpo}{\mathbb{E}_{\substack{\outp_1, \outp_2 \sim \prior{\outp} }}}

%% file: icml2024/sections/abstract.tex

Distribution matching methods for language model alignment such as Generation with Distributional Control (GDC) and Distributional Policy Gradient (DPG) have not received the same level of attention in reinforcement learning from human feedback (RLHF) as contrastive methods such as Sequence Likelihood Calibration (SLiC), Direct Preference Optimization (DPO) and its variants. We identify high variance of the gradient estimate as the primary reason for the lack of success of these methods and propose a self-normalized baseline to reduce the variance. We further generalize the target distribution in DPG, GDC and DPO by using Bayes' rule to define the reward-conditioned posterior. The resulting approach, referred to as \shortname\ - \longname\ acts as a bridge between distribution matching methods and DPO and significantly outperforms prior art in summarization and Antropic HH tasks.

%% file: icml2024/sections/introduction.tex
\section{Introduction}
\label{sec:intro}

Reinforcement Learning from Human Feedback (RLHF) has emerged as a pivotal technique to fine-tune Large Language Models (LLMs) into conversational agents that obey pre-defined human preferences \cite{ouyang2022training, bai2022training, openai2023gpt4, touvron2023llama}. This process involves collecting a dataset of human preferences and using it to align a Supervised Fine-Tuned (SFT) model to human preferences.

Proximal Policy Optimization (PPO) RLHF \cite{ziegler2019pporlhf}, has been instrumental in the development of groundbreaking models such as GPT-3.5 \cite{ouyang2022training,ye2023comprehensive} and GPT-4 \cite{openai2023gpt4} among others. This RL technique trains a separate Reward Model (RM) to discriminate outputs based on human preferences.
The RM is then used to train a policy that maximizes the expected reward while regularizing it to not diverge too much from the SFT model. 

Despite its clear success, PPO has lately been displaced by offline contrastive techniques that are more scalable and do not make full use of a separate RM.  
Techniques like Likelihood Calibration (SLiC) \cite{zhao2023slic} or Rank Responses to align Human Feedback (RRHF) \cite{yuan2023rrhf} only keep ranking information from a RM. Direct Preference Optimization (DPO) \cite{rafailov2023direct}, which is currently the de-facto method used to align high-performing models such as Zephyr \cite{tunstall2023zephyr}, Mixtral \cite{jiang2024mixtral} or LLaMa-3\footnotemark\footnotetext{\url{https://ai.meta.com/blog/meta-llama-3/}}, trains the policy directly on human preferences without the need of a separate reward model\footnotemark\footnotetext{Arguably, the more performant versions of DPO still depend on a RM for determining preference data \cite{liu2023rso}.}.

Both PPO and DPO are derived from KL-controlled reward maximization \cite{jaques2017tuning}, which has a well known closed form solution \cite{levine2018reinforcement}. This optimal policy comes however in the form of an energy-based model (EBM) with an intractable partition function. A set of less well-known methods for RL in language models use distribution matching to align an SFT model to this EBM~\cite{khalifa2020gdc, parshakova2019distributional, korbak2022distributionmatching}. 
During the alignment of an SFT model via distribution matching, we need to sample from the target EBM. However, sampling from the target EBM is challenging, and hence distribution matching approaches sample from a proposal distribution instead, and reweigh the samples based on their importance weights. Despite the clear intuition behind distribution matching, these methods are not used commonly for the task of reinforcement learning from human feedback.


In this work, we address the primary reason for the lack of success of distribution matching methods for RLHF. Towards this end, we propose \shortname - \longname\, that extends the distribution-matching methods in two significant ways. Firstly, we propose a Bayesian approach to construct the target distribution for distribution matching. Specifically, we treat the SFT model as the prior distribution over the outputs for a given input, that is, $p(y|x)$. The likelihood $p(G=1|x,y)$ captures the goodness of an output for a given input and is defined as a function of the reward $r(x,y)$ based on the reward-modelling assumptions. The resulting reward-conditioned posterior $p(y|x,G=1)$, obtained using Bayes' rule, is chosen as the target distribution. When the underlying preference model behind the reward function $r(x,y)$ follows Bradley-Terry assumptions, we show that the posterior corresponds to the optimal policy of the KL-regularized reward-maximization objective used in PPO-RLHF~\cite{ziegler2019pporlhf}.

More significantly, we observe that the distribution matching technique suffers from high-variance of the gradient estimate, despite the baseline proposed in~\citet{korbak2022distributionmatching}. In this work, we propose a self-normalized baseline that significantly reduces the variance of the gradient estimate. We prove that the resulting estimate corresponds to the gradient of a novel self-normalized KL divergence objective for distribution matching. Furthermore, self-normalization in the baseline helps us to establish DPO-sft (a version of DPO where samples are generated from the SFT model and scored by the reward model) as a special case of the \shortname\ objective. 

In our experiments on TL;DR summarization \cite{stiennon2020learning,volske2017tldr} and AntropicHH \cite{bai2022training},
\shortname\ establishes a new state-of-the-art by significantly outperforming existing SoTA RL methods DPO and RSO~\cite{rafailov2023direct}.
In addition, we bridge the gap between \shortname\ and DPO by careful augmentation of DPO objective in two ways. Specifically, we incorporate:
1) multiple outputs for a given input prompt, and 2) reward as an importance weight in the DPO objective. 

Overall, we make the following contributions:



\begin{compactenum}
\item We propose a Bayesian approach to construct the target distribution in distribution matching methods of RL for LLMs. The resulting posterior generalizes the PPO optimal policy.
\item For distilling the posterior in our training policy, we propose a self-normalized baseline for variance reduction of the gradient estimate of the objective. The resulting algorithm is referred to as \shortname\ and the gradient estimate is referred to as the \shortname\ gradient estimate.
\item We theoretically prove that the proposed gradient estimate is an unbiased estimator of a modified form of KL divergence that we name \shortname~objective.
\item We derive the exact form of the \shortname~objective under Bradley-Terry preference model assumption. We also show DPO can be derived as a special case of the \shortname~objective.
\item Finally, we empirically substantiate our claims by experimenting on two natural language generation tasks.
\end{compactenum}

%% file: icml2024/sections/relatedworks.tex
\section{Related Works}
In relation to RLHF approaches, InstructGPT \cite{ouyang2022training} made fundamental contributions to conversational agent alignment and showed how Proximal Policy Optimization (PPO) \cite{schulman2017proximal} could be used for this purpose. PPO is however an online-RL algorithm, which carries high costs of sampling and keeping additional LLMs in memory, such as value networks.

After PPO, offline-RL algorithms emerged that are simpler and have lower computational costs. SLiC \cite{zhao2023slic} proposed a margin loss between preferred and rejected outputs, regularized with a SFT loss, RRHF \cite{yuan2023rrhf} extends this idea to multiple outputs. Direct Preference Optimization (DPO) \cite{rafailov2023direct} starts from the KL-controlled reward maximization objective as PPO and derives an analytical form for a reward model based on the optimal policy for this objective. Once plugged into the standard parametrization of a Bradley-Terry reward model, this yields an objective without an explicit reward and results in a contrastive gradient update rule. DPO is well funded theoretically and has found clear empirical success in aligning LLMs \cite{tunstall2023zephyr,ivison2023camels,jiang2024mixtral}. It has also a large amount of follow-up works. Statistical Rejection Sampling Optimization (RSO) \cite{liu2023rso} proposes sampling from the optimal distribution and use a reward model for labeling, Identity Preference Optimization (IPO) \cite{azar2024general} introduces additional regularization, Kahneman-Tversky Optimization (KTO) \cite{ethayarajh2024kto} derives a similar loss that does not require preference pairs and Odds Ratio Preference Optimization (ORPO) \cite{hong2024reference} enriches the SFT loss with a term based on the odds of producing the desired response.

As mentioned in Section~\ref{sec:intro}, PPO and DPO originate from the same KL-controlled reward maximization which has an EBM as its optimal policy solution. Distributional Policy Gradient (DPG) \cite{parshakova2019distributional} proposes using importance sampling to learn a policy matching the distribution of an EBM via KL minimization. DPG notes that, for stability purposes, offline optimization is needed. Generation with Distributional Control (GDC) \cite{khalifa2020distributional} proposes the application of DPG to controlled language generation, introduces a KL threshold to update of the offline proposal distribution and makes a connection with maximum entropy methods. GDC++ \cite{korbak2022distributionmatching} shows that, similarly to regular policy gradient, variance reduction increases performance. It also shows that distributional matching of the optimal distribution optimizes the same objective as KL-controlled reward maximization.

Similar to the above approaches, \shortname\ uses distribution matching to train the policy. However, it differs from GDC and GDC++ in two major aspects. Firstly, the target distribution in \shortname\ is the reward-conditioned posterior derived using Bayes' rule. Secondly, and more importantly, \shortname\ uses a self-normalized baseline which results in significant variance reduction of the gradient estimate and hence, it tremendously improves performance. The self-normalized baseline yields a connection to DPO~\cite{rafailov2023direct}, showing that DPO-sft (a variant of DPO where the samples come from base/SFT policy and are scored using a reward model) is a special case of \shortname.

Related to learning distributions conditioned on desired features, earlier works such as \citet{ficler-goldberg-2017-controlling}, train special feature embeddings to produce text with desired target features. More recently \citet{chen2021decision} conditions on a goodness token, and
\cite{lu2022quark,korbak2023pretraining} threshold a reward model for the same purpose. 
Although inspired by reward conditioning, \shortname~deviates from these approaches by not explicitly parametrizing the conditional distribution that takes both prompt $\inp$ and desired reward as input. Instead, \shortname\ poses the problem as distributional matching of the posterior distribution.

%% file: icml2024/sections/technique.tex
\section{Notation}

Let $\inp$ be an input prompt and $\yset$ be the set of all output sequences.
Let $\prior{\outp}$ be the conditional probability assigned by a supervised fine-tuned (SFT) language model (LM) to an output $\outp \in \yset$ for an input $\inp$. 
Let $\rewardfn{\outp}$ be the corresponding reward value assigned by a given reward function.
Further, let $\rewardrv$ represent a binary random variable 
such that the probability
$\rewardprob{\outp}$ captures the goodness of a given output $\outp$ for a given input $\inp$.
The relationship between probability $\rewardprob{\outp}$ and reward value $\rewardfn{\outp}$ depends upon the modelling assumptions made while training the reward model. In \cref{sec:btm}, we illustrate the connection between $\rewardprob{\outp}$ and $\rewardfn{\outp}$
for both absolute and relative reward models such as Bradley-Terry.

Given the prior $\prior{\outp}$, and the goodness model $\rewardprob{\outp}$, 
we define the posterior $\posterior{\outp}$ as our desired distribution that assigns high probability to `good' outputs.
To mimic sampling from the posterior distribution, we aim to train a new model $\policy{\outp}$ by minimizing the KL divergence between the two.


\section{Approach}

\textbf{Bayesian reformulation:}
We first use Bayes' rule to represent the reward--conditioned posterior as:
\begin{equation}\label{eq:posterior}
    \posterior{\outp} =  \frac{\prior{\outp} \rewardprob{\outp}}{\rewardprobx}
\end{equation}
We are interested in sampling from $\posterior{\outp}$, \ie, samples with high reward. An obvious solution is to sample from $\prior{\outp}$ and keep rejecting the samples until a sample with a high reward is obtained.
Despite its simplicity, rejection sampling is expensive. Hence, in this work, we propose to learn a distribution, $\policy{\outp}$, that can mimic sampling from the posterior without the computation overhead of rejection sampling.

\textbf{Training objective:} To train the parameters $\theta$, we propose to minimize the KL--divergence between the posterior distribution $\posterior{\outp}$, and $\policy{\outp}$. 
This is equivalent to maximizing the objective:
\begin{align}
    \loss &= - \expectation_{\posterior{\outp}} \left[ \log \frac{\posterior{\outp}}{\policy{\outp}} \right]
\end{align}
By collecting all the constant terms with respect to $\theta$ in $C$, the objective can be written as
\begin{align}\label{eq:objective}
    \loss &=  \expectation_{\posterior{\outp}} \left[ \log \policy{\outp} \right] +C\,
\end{align}

Henceforth, the constant $C$ will be omitted in subsequent formulations of the objective $\loss$.

\textbf{Approximation with importance sampling:}\label{sec:approx_is}
To empirically compute the expectation in \cref{eq:objective}, we need to sample from the posterior $\posterior{\outp}$. Unfortunately, it is not possible to do so directly, and hence we resort to importance sampling \cite{tokdar2010importance},
\begin{align}\label{eq:iw_objective}
    \loss &=  \expectation_{\proposal{\outp}} \left[ \frac{\posterior{\outp_i}}{\proposal{\outp_i}} \log \policy{\outp} \right]\,
\end{align}
where $\proposal{\outp}$ is an an easy--to--sample proposal distribution. Taking into account the Bayes rule in equation \eqref{eq:posterior}  we approximate the expectation by a sample average of $n$ outputs $(\outp_1, \ldots \outp_n)$ from $\proposal{\outp}$. We further use self--normalized importance sampling (ch. 9 in \citet{owen2013monte}), normalizing the weights by their sum. This results in the following loss for a given $\inp$:
\begin{align}\label{eq:approx_is}
&\lossestimate = \sum_{i=1}^n \alphahaty{i} \log \policy{\outp_i} \text{, where } \outp_i \sim \proposal{\outp} \\
&\alphahaty{i} = \frac{\alphay{i}}{\sum_{j=1}^n \alphay{j}} \text { , } \alphay{i} = \frac{\prior{y_i}}{\proposal{y_i}} \rewardprob{\outp_i} \label{eq:alpha}
\end{align}
Note that we dropped $\rewardprobx$ from $\alphay{i}$ as it will get cancelled due to self--normalization. We have added subscripts to $\alpha$ to show that they depend on the samples $\outp$.
The gradient of the objective can be written as
\begin{equation} \label{eq:gradient_is}
    \nabla_\theta \lossestimate = \sum_{i=1}^n \alphahaty{i} \nabla_\theta \log \policy{\outp_i}
\end{equation}
\textbf{Baseline subtraction:} One critical issue with the loss in \cref{eq:approx_is} is that it assigns a positive weight to all the samples $\outp_i$ for a given $\inp$, irrespective of its reward distribution $\rewardprob{\outp_i}$. In other words, the model is trained to increase the probability of all the samples and not just the high-reward ones. This is not an issue when all the samples have high reward, that is, the proposal distribution is the same as the posterior $\proposal{\outp} = \posterior{\outp}$.  

When the proposal is away from the posterior, we hypothesize that it is crucial for low-reward samples to have negative weights. In GDC++ \cite{korbak2022distributionmatching}, the authors proposed to subtract the following baseline from its gradient estimate:
\begin{equation}
    Baseline = Z\frac{\policy{\outp}}{\proposal{\outp}} \nabla_\theta \log \policy{\outp}\,
\end{equation}
where $Z$ is the normalization constant of the target EBM (target posterior in this paper). In contrast, we propose a self-normalized baseline as described below:


To obtain a baseline for our case, we first note the following general result (derived for $\policy{\outp}$ here),
\begin{align} \label{eq:baseline}
     & \expectation_{\policy{\outp}}  \nabla_\theta \log \policy{\outp} 
      = \expectation_{\policy{\outp}}  \frac{1}{\policy{\outp}}\nabla_\theta \policy{\outp} \\
     & = \sum_{\outp \in \yset} \nabla_\theta \policy{\outp} 
      = \nabla_\theta \sum_{\outp \in \yset} \policy{\outp} = \nabla_\theta 1 = 0
\end{align}
Thus, this expectation can be subtracted from the gradient in \eqref{eq:gradient_is} without introducing any bias. 
To estimate the baseline, we reuse the same samples $(\outp_1, \ldots, \outp_n)$ that are used in \eqref{eq:approx_is} and apply self--normalized importance sampling to get:
\begin{align} \label{eq:baseline_is}
&\expectation_{\policy{\outp}} \nabla_\theta \log \policy{\outp} \approx \sum_{i=1}^n \betahatthetay{i} \nabla_\theta \log \policy{\outp_i} \\
&\outp_i \sim \proposal{y} \text{ and } \betahatthetay{i} = \frac{\betathetay{i}}{\sum_{j=1}^n \betathetay{j}} \text{ , } \betathetay{i} = \frac{\policy{y_i}}{\proposal{y_i}}
\label{eq:beta}
\end{align}
Subtracting the baseline estimate from \cref{eq:gradient_is}, we get:
\begin{align}\label{eq:grad_brain}
    \nabla_\theta \lossestimate  
    = \sum_{i=1}^n \left(\alphahaty{i} - \betahatthetay{i}\right) \nabla_\theta \log \policy{y_i}
\end{align}
$\outp_i, \alphahaty{i},\text{and } \betahatthetay{i}$ are same as in \cref{eq:alpha,eq:beta}. 
We call the self-normalized baseline subtracted gradient estimate in \eqref{eq:grad_brain} the \textbf{\shortname\ gradient estimate}.
Intuitively, $\alphay{i}$ and $\betathetay{i}$ are proportional to the posterior $\posterior{\outp_i}$ and policy $\policy{\outp_i}$ distributions, respectively.
Thus, the weight (difference of normalized $\alphay{i}$ and $\betathetay{i}$) of each $\outp_i$  captures how far the current estimate of policy  $\policy{\outp_i}$ is from the true posterior $\posterior{\outp_i}$.
This also alleviates the critical issue of assigning positive weights to all $\outp_i$ irrespective of their reward.
Samples with lower reward, and hence lower posterior $\posterior{\outp_i}$ (consequently lower $\alphay{i}$), will get negative weights as soon as the policy distribution $\policy{\outp_i}$ assigns higher probability to them (consequently higher $\betathetay{i}$) than the posterior, and vice-versa.

\input{icml2024/algorithms/brain}
The \shortname\ algorithm with the self-normalized baseline is given in \Cref{alg:brain}. 
We initialize both our proposal $\proposal{y}$ and policy $\policy{y}$ with $\prior{y}$. 
We update the proposal $e$ times after every $k$ gradient updates.
$m$ is the number of prompts sampled from dataset $D$ after updating the proposal, and $n$ is the number of outputs generated for each prompt $x \in D$.

In our experiments, we show that subtracting the self--normalized baseline results in a large improvement in performance. We hypothesize that the baseline allows the model to focus on distinctive features of high-reward outputs compared to the lower-reward ones.
A formal justification of the resultant \shortname~gradient estimate is provided in the Appendix~\ref{sec:snkl}.

\subsection{A formal justification of the \shortname\ gradient estimate:}\label{sec:snkl}
Self-normalized importance sampling (SNIS) introduces bias in any estimator. In this paper, we have used SNIS to approximate the importance weights as well as the baseline in our gradient estimate. Using biased gradient estimators for training often results in optimization of an objective different from the desired one.

However, in our case, we show that using the \shortname\ gradient estimate for training, results in minimizing
a self-normalized version of KL--divergence (defined below) whose minimum value is achieved only when $\policy{\outp} = \posterior{\outp}$. Note that this is a consequence of the chosen self-normalized baseline in \eqref{eq:beta}.

First, we define self-normalized KL divergence in the context of this paper. Next, we show that our proposed \shortname\ gradient estimate is an unbiased estimator of the gradient of this divergence measure. 
Finally, we prove that this self-normalized KL divergence is non-negative and equals $0$ only when the policy learns to mimic the posterior. The proofs are in \cref{sec:proofs}.
\begin{restatable}{definition}{selfkl}
\label{def:self_kl}
Let the proposal distribution $\proposal{\outp}$, training policy $\policy{\outp}$ and posterior $\posterior{\outp}$ be as defined earlier. Furthermore, we assume that $\text{support}(\pposterior) \subseteq \text{support}(\ppolicy) \subseteq \text{support}(\pproposal)$. For any $n$ outputs, $Y_x=(\outpp_1, \ldots, \outpp_n)$, let 
$\alphahaty{i}$ and $\betahatthetay{i}$ be the self--normalized importance sampling weights for the loss and the baseline, respectively (\cref{eq:alpha,eq:beta}). 
The self-normalized KL--divergence between the posterior and training policy for the given proposal distribution is defined as:
\begin{align}
    &\KLnorm  ( \posterior{\outp}  || \policy{\outp} ) \notag \\
    &= \multiexpectation \left[ \KL( \alphahatyx || \betahatthetayx ) \right] \text{ where } \label{eq:snkl} \\
   &\KL(\alphahatyx || \betahatthetayx) = \sum_{i=1}^n \alphahaty{i}\log \frac{\alphahaty{i}}{\betahatthetay{i}} \label{eq:alphabetakl}
\end{align}
\end{restatable}

\begin{restatable}{theorem}{unbiasedestimator}
\label{th:unbiased_estimator}
    The \shortname\ gradient estimate defined in \eqref{eq:grad_brain} is an unbiased estimator of the gradient (\wrt\ $\theta$) of negative self-normalized KL--divergence between the posterior $\posterior{\outp}$ and training policy $\policy{\outp}$ defined in \eqref{eq:snkl}. Here, the dependence of KL divergence on $\theta$ comes from $\betahatthetay{i}$ being a function of $\policy{\outp_i}$.
\end{restatable}
Since the gradient of our objective in \cref{eq:grad_brain} is the same as the gradient of negative self-normalized KL-divergence defined in \cref{eq:snkl}, in the rest of the paper, we refer to negative self-normalized KL-divergence between the posterior and the training policy as \textbf{\shortname\ objective}.
\begin{restatable}{theorem}{sameminima}
\label{th:same_minima}
The self-normalized KL--divergence defined in \eqref{eq:snkl} reaches its minimum value of $0$ if and only if the KL--divergence between the posterior and training policy also reaches $0$. 
Also, $\policy{\outp}=\posterior{\outp}$ is the only minima of self-normalized KL--divergence defined in \eqref{eq:snkl}.
\end{restatable}
\section{Connection with existing RLHF methods}
In this section, we show that the proposed \shortname\ algorithm acts as a bridge between two disparate objectives used for reinforcement learning in LLMs as shown in Figure~\ref{fig:connections}.
\begin{enumerate}
    \item Distribution matching objectives as described in \cite{korbak2022distributionmatching,parshakova2019distributional,khalifa2020gdc}.
    \item The contrastive-training approach presented in DPO~\cite{rafailov2023direct}, specifically DPO-sft where the samples come from the base/SFT policy and are scored using the reward model.
\end{enumerate}

\begin{figure*}[t]
    \centering
    \includegraphics[width=0.8\textwidth]{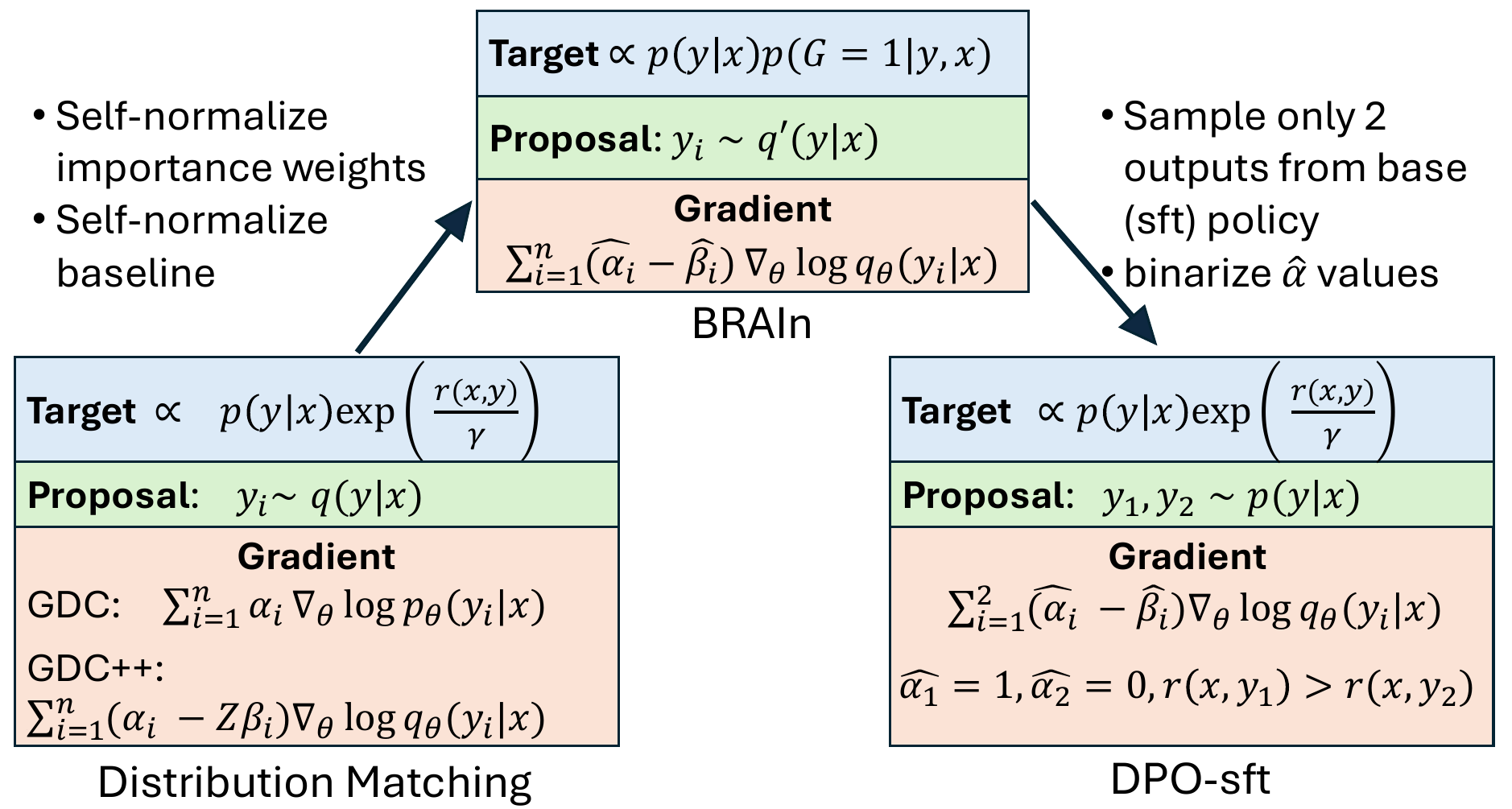}
    \caption{\shortname\ acts as a bridge between distribution matching methods (GDC~\cite{khalifa2020gdc} and GDC++~\cite{korbak2022distributionmatching}) and DPO~\cite{rafailov2023direct}, specifically DPO-sft where the samples come from the base/SFT policy. The values $\alpha_i, \hat{\alpha_i}$ and $\hat{\beta_i}$ are as defined in equations \eqref{eq:alpha} and \eqref{eq:beta} whereas $Z$ is the normalization constant of the target. Note that the proposal distribution in the distribution matching methods and \shortname\ is chosen differently.}
    \label{fig:connections}
\end{figure*} 

In the rest of the section:
\begin{enumerate}
    \item We establish the connection between the target distribution of \shortname\ as defined in \eqref{eq:posterior}, and the PPO-optimal policy which is the target distribution for PPO, DPO as well as distribution matching methods for RLHF.  Specifically, we establish that under Bradley-Terry preference modelling assumptions, the posterior becomes equal to the PPO-optimal policy. 
    \item We derive the DPO-sft gradient estimate as a special case of the \shortname\ gradient estimate.
\end{enumerate}

\subsection{Posterior and PPO-optimal policy}
In this section, we show that for a Bradley-Terry preference model, the posterior $\posterior{\outp}$ corresponds to the PPO-optimal policy. Before we jump into the proof, we briefly describe Bradley-Terry preference model and its connection with the goodness model in \shortname\ $\rewardprob{\outp}$.


\subsubsection{Bradley-Terry Preference Model for LLMs}\label{sec:btm}
Given an absolute goodness measure $\goodness_i$ for each $\outp_i$, Bradley-Terry Model (BTM) defines the probability of choosing $\outp_i$ over $\outp_j$ as:
\begin{equation} \label{eq:btm_orig}
    \bigprob(\outp_i \succ \outp_j) = \frac{\goodness_i}{\goodness_i + \goodness_j}
\end{equation} 
Recall that in our formulation, the binary random variable $\rewardrv$ represents the goodness of an output, and hence, the conditional probability $\rewardprob{\outp_i}$ can be used as a proxy for $\goodness_i$:
\begin{equation}\label{eq:btm}
    \bigprob(\outp_i \succ \outp_j|\inp) = \frac{\rewardprob{\outp_i}}{\rewardprob{\outp_i} + \rewardprob{\outp_j}}
\end{equation}
During training of the reward function $\rewardfn{\outp}$, we are given triplets $(\inp, \outp_i, \outp_j)$ such that for an input $\inp$, the response $\outp_i$ is preferred over the response $\outp_j$. 
We train it by maximizing the log-likelihood over all the training triplets.
During training, the goodness measure $g_i$ is parameterized as $\exp{(\frac{\rewardfn{\outp_i}}{\gamma})}$ resulting in the following log-likelihood of a triplet:
\begin{equation}\label{eq:btm_param}
    \log \bigprob(\outp_i \succ \outp_j| \inp) = \log \sigma\left(\frac{\rewardfn{\outp_i} - \rewardfn{\outp_j}}{\gamma}\right)
\end{equation}
Thus, by equating \cref{eq:btm_param} and log of \cref{eq:btm}, we get the maximum likelihood estimate of the log--odds as:
\begin{equation}\label{eq:btm_equate}
    \log \frac{\rewardprob{\outp_i}}{\rewardprob{\outp_j}} =  \frac{\rewardfn{\outp_i} - \rewardfn{\outp_j}}{\gamma}
\end{equation} 
The above ratio is exactly what we need to compute self--normalized importance weights $\alphay{i}/\sum_{j=1}^{n}\alphay{j}$ in \cref{eq:approx_is}.
We formalize this result in the proposition below.
\begin{restatable}{proposition}{bradleyterry}
\label{th:bradley_terry}
    For the Bradley-Terry preference model defined in \cref{eq:btm} and parameterized by \cref{eq:btm_param}, the self--normalized importance-weights $\alphahaty{i} = \alphay{i}/\sum_{j=1}^{n}\alphay{j}$ have the following form:
    \begin{align} \label{eq:imp_weights}
     \frac{\exp(\frac{\rewardfn{\outp_i}}{\gamma} + \log \prior{\outp_i} - \log \proposal{\outp_i})}{\sum_{j=1}^n \exp(\frac{\rewardfn{\outp_j}}{\gamma} + \log \prior{\outp_j} - \log \proposal{\outp_j})} 
    \end{align}
    In the special case where the proposal distribution $\proposal{\outp}$ is the same as the prior distribution $\prior{\outp}$, the importance weights reduces to
    \begin{equation}\label{eq:imp_weighted_sp}
        \alphahaty{i} =  \frac{\exp \left( \frac{\rewardfn{\outp_i}}{\gamma} \right)}{\sum_{j=1}^n \exp \left( \frac{\rewardfn{\outp_j}}{\gamma} \right)}
    \end{equation}
\end{restatable}
See the appendix for the proof.
Observe that $\alphahaty{i}$ in \cref{eq:imp_weighted_sp} is nothing but a softmax over the reward values $\rewardfn{\outp_i}$.
In the next section, we show that DPO is a special case of our formulation when we replace softmax with argmax and set $n=2$.
This amounts to assigning all the importance weight to the output sample with the most reward.

\begin{restatable}{theorem}{optimalpolicy}
For a Bradley-Terry reward model, the posterior $\posterior{y}$ is same as the PPO optimal policy given by:
\begin{equation}
\optimal{y} = \frac{\prior{\outp}\exp(\frac{\rewardfn{\outp}}{\gamma})}{\left(\sum_{\bar{\outp} \in \yset}\prior{\bar{\outp}}\exp\left( \frac{\rewardfn{\bar{\outp}}}{\gamma} \right)\right)}
\end{equation}
The above eqn. for PPO optimal policy is as shown in equation(4) in \citet{rafailov2023direct}. Note that the reference policy in DPO is same as the prior policy in our formulation.

\end{restatable}

\subsection{DPO-sft as a special case of \shortname} \label{sec:dpo_sft}
In DPO~\cite{rafailov2023direct}, an ideal setting is discussed where the samples are generated from the base policy ($\prior{\outp}$ in our case) and annotated by humans. Often, the preference pairs in publicly available data are sampled from a different policy than $\prior{\outp}$. 
To address this, \citet{liu2023rso} experiment with a variant of DPO, called \dposft, in which a reward model is first trained on publicly available preference pairs and then used to annotate the samples generated from the base policy $\prior{\outp}$. 


We claim that \shortname\ reduces to \dposft\ when we generate only 2 samples per input $\inp$ and assign the importance weight of the winner to 1.
The last assumption is equivalent to replacing softmax in \cref{eq:imp_weighted_sp} with an argmax.
We formalize it in the following theorem:

\begin{restatable}{theorem}{dpospecialcase}
\label{thm:dpo_special_case}
Let the proposal distribution $\proposal{y}$ in \shortname\ be restricted to the prior $\prior{y}$.
When $n=2$, and the softmax is replaced by argmax in
\cref{eq:imp_weighted_sp}, then the \shortname\ objective reduces to DPO objective proposed by \citeauthor{rafailov2023direct}
\end{restatable}
\begin{proof}
We start with \shortname\ objective defined in \cref{eq:snkl} and insert our assumptions to arrive at DPO objective.

By substituting $n=2$ and $\proposal{\outp}=\prior{\outp}$ in R.H.S. of \cref{eq:snkl}, and using \cref{eq:alphabetakl}, we get:
\begin{equation}\label{eq:dpo_sft_iw}
\multiexpectationdpo \left[ \alphahaty{1}\log \frac{\alphahaty{1}}{\betahatthetay{1}} + \alphahaty{2}\log \frac{\alphahaty{2}}{\betahatthetay{2}} \right]
\end{equation}
Now, without loss of generality, assume that $\rewardfn{\outp_1} > \rewardfn{\outp_2}$. Replacing softmax with argmax in \cref{eq:imp_weighted_sp}, we get $\alphahaty{1}=1$ and $\alphahaty{2}=0$.
Plug it in \cref{eq:dpo_sft_iw} to get:
\begin{equation*} - \multiexpectationdpo \log \betahatthetay{1}
\end{equation*}
Now recall from \cref{eq:beta} that $\betahatthetay{1} = \betathetay{1}/(\betathetay{1}+\betathetay{2})$ and $\betathetay{i}=\frac{\policy{\outp_i}}{\prior{\outp_i}}$. Replacing it in the above equation and rearranging the terms, we get:
\begin{equation*}
- \multiexpectationdpo \log \sigma \left(\log \frac{\policy{\outp_1}}{\prior{\outp_1}} - \log \frac{\policy{\outp_2}}{\prior{\outp_2}} \right)
\end{equation*}

The above expression is exactly same as equation (7) in \citeauthor{rafailov2023direct}.

\end{proof}

We also note that to extend DPO to multiple outputs, \ie, $n>2$, \citeauthor{rafailov2023direct} resorts to a more general Plackett-Luce preference model. 

%% file: icml2024/algorithms/brain.tex
\begin{algorithm}
\caption{\small Bayesian Reward-conditioned Amortized Inference}
\label{alg:brain}
\begin{algorithmic}[1]
\REQUIRE $r(x,y), p(y|x), D, e, m, n, k$
\STATE Initialize $q'(y|x) \leftarrow p(y|x)$, $q_\theta(y|x) \leftarrow p(y|x)$
\STATE Initialize $D_0 \leftarrow \{\}$
\FOR{$t = 1$ to $e$}
    \STATE $S_t \leftarrow$ randomly selected $m$ prompts from $D$
    \FORALL{$x \in S_t$}
        \STATE $Y_x \leftarrow$ generate $n$ samples using $q'(y|x)$ 
        \FORALL{$y \in Y_x$}
            \STATE Cache $\log q'(y|x), \log p(y|x), r(x,y)$ in $S_t$
            \STATE Cache normalized $\alphahaty{}$ (\cref{eq:alpha} or \cref{eq:imp_weights}) in $S_t$
        \ENDFOR
    \ENDFOR
    \STATE $D_t \leftarrow D_{t-1} \cup S_t$
    \FOR{$j = 1$ to $k$}
        \STATE $B \leftarrow $ randomly sampled batch from $D_t$
        \FORALL{$(x, Y_x, q'(y|x), \alphahaty{}) \in B$}
            \STATE Compute $\betahatthetay{}$ (\cref{eq:beta})
            \STATE $\loss \leftarrow \sum_{y \in Y_x} (\hat{\alpha}_y - \hat{\beta}_y) \log \policy{y}$
        \ENDFOR
        \STATE Update $\theta$ using $\nabla_\theta\sum_{x\in B}\loss$
    \ENDFOR
    \STATE $\proposal{y} \leftarrow \policy{y}$
\ENDFOR
\STATE \textbf{Return} $\ppolicy$

\end{algorithmic}
\end{algorithm}

%% file: icml2024/sections/experiments.tex
\section{Experimental Setup}

\textbf{Tasks:}
We conduct experiments on two natural language generation tasks, \viz, summarization, and multi-turn dialog.
In the summarization task, we align CarperAI's summarization LM to a Bradley-Terry reward model. 
We train and evaluate using Reddit TL;DR dataset,
in which input is a post on Reddit, and the aligned model should generate a high reward summary.\\
In the multi-turn dialog task, we ensure that a open-ended chatbot's responses are \textbf{H}elpful and \textbf{H}armless. 
We use QLoRA-tuned Llama-7b \cite{dettmers2023qlora} as SFT model, 
and a fine-tuned GPT-J \cite{gpt-j} as the reward model. 
We train and evaluate using a subset of AntrophicHH \cite{bai2022training} dataset.
See \cref{sec:modeldatasetlinks} for hyperlinks to all the datasets and models described above.

\textbf{Evaluation metrics:} We use win--rate over gold responses and direct win--rate against the baseline responses to measure the performance of various techniques.
Win--rate is defined as the fraction of test samples on which the generated response gets a higher reward than the gold response.
We use two independent reward functions to compute the win--rate: (1) \textit{Train RM}, which is the reward function used to align the SFT model, and 
(2) \textit{LLM eval}, which follows LLM-as-a-judge \cite{zheng2023judging} and prompts a strong instruction following LLM, Mixtral-8x7B \cite{jiang2024mixtral}, to compare the two outputs and declare a winner.  The first metric captures the effectiveness of the alignment method in maximizing the specified reward, while a high win--rate using the LLM prompting ensures that the alignment method is not resorting to reward-hacking \cite{skalse2022defining}.

\input{icml2024/tables/performance_comparison}

We prompt GPT-4 \cite{openai2023gpt4} to directly compare \shortname's responses with each of the baselines separately. 
See \cref{sec:prompt-for-llm-as-judge} for the specific instructions used for prompting GPT-4 and Mixtral-8x7B.


\textbf{Training details:}
While DPO uses human-annotated data directly, we generate $n=32$ samples per input prompt for the other models (\dposft, RSO, and \shortname). The samples are organized in $16$ pairs for \dposft\ and RSO, while for \shortname, the $32$ samples are grouped together. We use $\gamma=1$ for \shortname~ in all our experiments, unless otherwise stated, and average the log-probability of the tokens.
We train each model for a total of $10,000$ steps ($e=40$ and $k=250$ in \cref{alg:brain}) with 4 prompts and $32$ outputs per prompt in a batch $B$. 
The model is evaluated after every $1,000$ steps and the best model is selected based on the win rate against the gold response on the cross-validation set. Other training details are given in Appendix~\ref{sec:other_train_details}.

\section{Experimental Results}

\textbf{Research Questions:} Our experiments aim to assess \shortname's performance in aligning an SFT model to a given reward function and compare it against various baselines. Specifically, we answer the following research questions:
\begin{compactenum}
    \item How does \shortname~compare to existing baselines?
    \item How does the KL-reward frontier of \shortname\ compare with that of DPO?
    \item What is the role of self-normalized baseline subtraction in \shortname?
    \item How to bridge the gap between DPO and \shortname?
    \item How does varying the number of output samples per input affect \shortname's performance?
\end{compactenum}

\subsection{Comparison with baselines}

\Cref{tab:perform_comparison} compares the win--rate of \shortname\ against our baselines -- RSO, DPO, and DPO-sft. We observe that \shortname\ consistently outperforms all the baselines, irrespective of the model used to compute win--rates.
Specifically, when measured using the specified reward model, \shortname\ achieves 8 and 4 pts better win--rate than the strongest baseline on AntrophicHH and TL;DR, respectively.
Next, we observe that even though \dposft\ has a much better win--rate than DPO when computed using the specified reward function, their performance is the same when judged by an independent LLM, indicative of reward--hacking.

\input{icml2024/tables/headon_comparison}
In \Cref{tab:headon_comparison}, we summarize the results of head-to-head comparisons of \shortname\ with each of the baselines judged by GPT-4 on a set of 500 examples from each of the dataset. \textit{Win \%} indicates the percentage of examples for which \shortname\ was declared the winner compared to the baseline. We observe that \shortname\ wins twice as many times as the baselines on AnthropicHH.

\subsection{KL-reward frontier}
Next, we compare the KL-reward frontier of \shortname\ with that of DPO-sft by varying the value of $\gamma$ ($\beta$ in the case of DPO-sft) and selecting multiple checkpoints for each $\beta/\gamma$. For each checkpoint, we sample $1000$ prompts and generate $8$ outputs per prompt. For each output, we compute the reward $\rewardfn{\outp}$ and $\log \policy{\outp} - \log \prior{\outp}$ and average them over the outputs and the prompts. This gives us $(KL, reward)$ for each checkpoint. We plot these values in Figure~\ref{fig:kl_reward}.

\begin{figure}[t]
    \centering
    \includegraphics[width=\columnwidth]{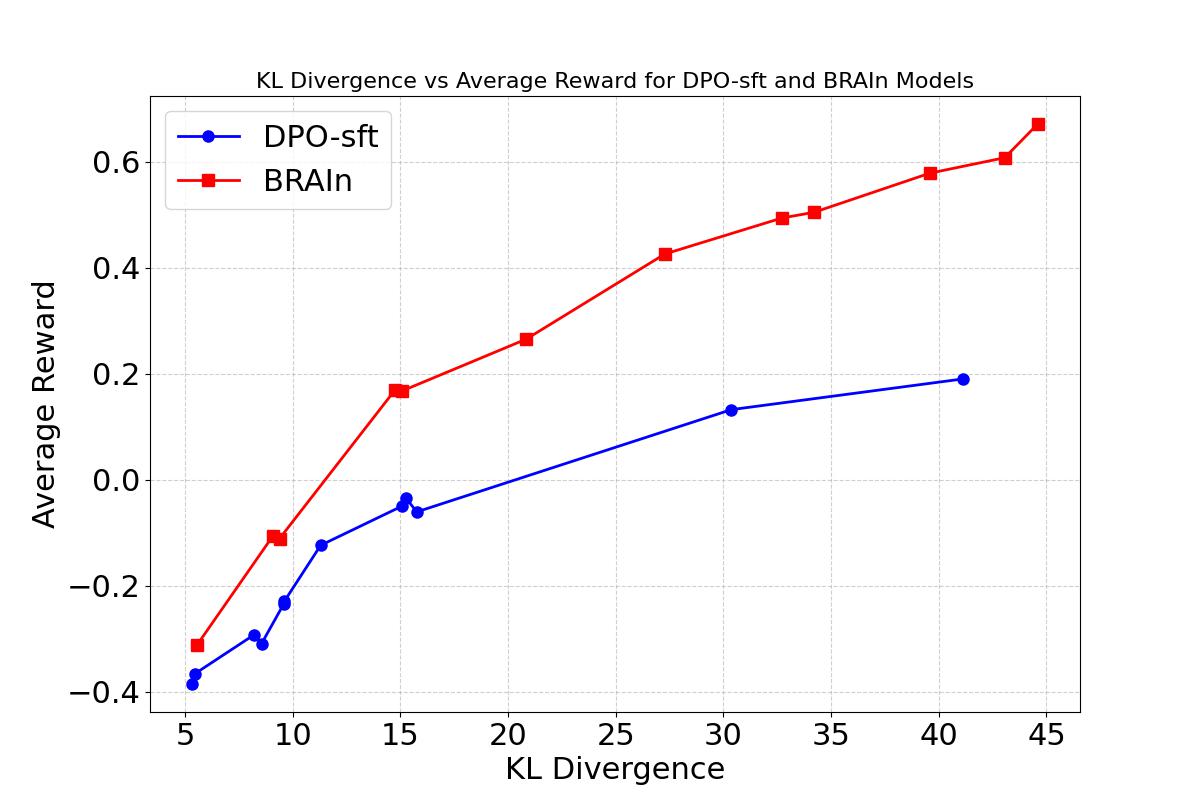}
    \caption{The KL-reward frontier of \shortname\ and DPO-sft}
    \label{fig:kl_reward}
\end{figure} 

As can be observed from the plot, for the same KL divergence, \shortname\ achieves a much higher reward than DPO-sft. Moreover, DPO-sft fails to achieve a reward as high as \shortname\ no matter how much the beta value is decreased.

\subsection{Role played by self-normalized baseline}
Next, we demonstrate the crucial role that the self-normalized baseline plays in \shortname. 

Our first experiment is a toy experiment where the posterior/target is defined to be the 1-D standard Gaussian distribution $\mathcal{N}(0, 1)$. The training policy $q_\theta$ is also Gaussian $\mathcal{N}(1, 1)$. We generate samples from the proposal distribution which is also chosen to be Gaussian $\mathcal{N}(\theta, 1)$ where $\theta$ is varied from $0$ to $1$. For each value of $\theta$, we generate $8$ samples from the proposal distribution to compute the GDC, GDC++ \cite{korbak2022distributionmatching, khalifa2020gdc}, and \shortname\ gradient estimate. We repeat this experiment $2000$ times and compute the variance across the different sets. The results are plotted in Figure~\ref{fig:variance_grad}.

\begin{figure}[t]
    \centering
    \includegraphics[width=\columnwidth]{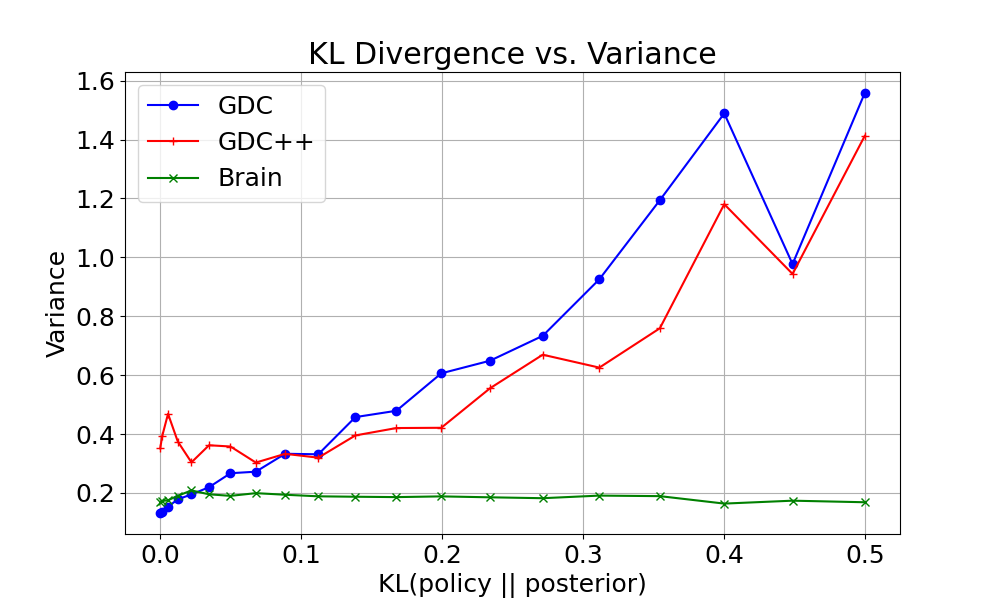}
    \caption{Variance of GDC, GDC++ and \shortname\ gradient estimates}
    \label{fig:variance_grad}
\end{figure} 

\input{icml2024/tables/self_normalized_baseline}

Next, we demonstrate the effect of self-normalized baseline on the performance of \shortname\ on TL;DR and Anthropic datasets. Table~\ref{tab:self_norm} summarizes our observations about the role that self-normalization of baseline plays in performance. As can be observed, there is a drastic reduction in performance if self-normalization in the baseline is removed.

\subsection{Bridging the gap between \dposft\ and \shortname}
\label{ssec:bridging-dposft-gap}
\input{icml2024/tables/dpo_brain_gap}
In \cref{sec:dpo_sft}, we show that under certain restrictions, \shortname\ objective reduces to \dposft\ objective. In this section, we start with \dposft\ and relax these restrictions one at a time to demonstrate the impact of each restriction.
First, we get rid of argmax and reintroduce softmax over rewards (\cref{eq:imp_weighted_sp}) to compute self--normalized importance weights $\alphahaty{i}$. This corresponds to the objective in \cref{eq:dpo_sft_iw}. We call this \dposft+IW.
Next, we relax the assumption of $n=2$, and take softmax over $n=32$ outputs instead of softmax over two samples in 16 pairs. We call it \dposft+IW+n. 
Finally, we relax the restriction of proposal distribution to be the prior distribution only and instead update our proposal periodically (\cref{alg:brain}) to arrive at \shortname.

\Cref{tab:bridging_gap} compares the win--rate of the two intermediate models (\dposft+IW, and \dposft+IW+n) with \dposft\ and \shortname\ over the AntropicHH dataset.
We first observe that relaxing each restriction consistently improves both the win--rates.
The biggest gain ($\sim4$ pts) comes by relaxing the assumption of $n=2$ and taking a softmax over all 32 outputs. This is expected as information contained in simultaneously comparing 32 outputs is potentially more than only 16 pairs.


\subsection{Effect of the number of output samples}
Next, we study the effect of varying the number of output samples ($n$) per input prompt $\inp$ on the performance of \shortname.
We retrain \dposft\ and \shortname\ on AntropicHH dataset for each $n \in \{2,4,8,16,32\}$. As done earlier, we create $n/2$ pairs from the $n$ samples while training \dposft.
The win--rates computed by specified reward model (Train RM) for each $n$ are plotted in \cref{fig:num_outputs}. 
We observe that including more samples in \shortname\ objective leads to improvement in performance till $n=8$ after which it saturates, whereas the performance of \dposft improves monotonically, albeit slowly.  
\input{icml2024/figures/num_outputs}


%% file: icml2024/tables/performance_comparison.tex
\begin{table}[]
\centering
\caption{Win--rate (in \%age) $\pm$ 95\% confidence--interval against the gold output on the test sets. Train RM corresponds to the reward model used during training, whereas for LLM eval., we prompt Mixtral-8x7B.}
\label{tab:perform_comparison}
\resizebox{\columnwidth}{!}{
\begin{tabular}{@{}lrrrr@{}}
\toprule
\multicolumn{5}{c}{\textbf{AnthropicHH}}                                 \\ \midrule
     & \textbf{DPO} & \textbf{DPO-sft} & \textbf{RSO} & \textbf{\shortname} \\ \midrule
\textbf{Train RM} & 54.60$\pm$1.37 & 87.37$\pm$0.91 & 84.59$\pm$0.99 & 95.40$\pm$0.57 \\
\textbf{LLM eval} & 67.02$\pm$1.29 & 66.68$\pm$1.29 & 67.22$\pm$1.29 & 74.36$\pm$1.20 \\ \midrule
\multicolumn{5}{c}{\textbf{Reddit TL;DR}}                                           \\ \midrule
\textbf{Train RM} & 86.72$\pm$0.82 & 90.86$\pm$0.70 & 91.24$\pm$0.68 & 95.21$\pm$0.52 \\
\textbf{LLM eval} & 60.26$\pm$1.18 & 60.55$\pm$1.18 & 60.41$\pm$1.18 & 64.74$\pm$1.16 \\ \bottomrule
\end{tabular}
}
\end{table}

%% file: icml2024/tables/headon_comparison.tex
\begin{table}[]
\centering
\footnotesize
\caption{Head-to-head comparison of the \shortname\ against the baselines using GPT-4.}
\label{tab:headon_comparison}
\resizebox{\columnwidth}{!}{
\begin{tabular}{@{}r|rrr|rrr@{}}
\multicolumn{1}{l|}{} & \multicolumn{3}{c|}{\textbf{AnthropicHH}} & \multicolumn{3}{c}{\textbf{Reddit TL;DR}} \\ \cmidrule(l){2-7} 
\multicolumn{1}{l|}{} & \textbf{Win \%} & \textbf{Tie \%} & \textbf{Loss \%} & \textbf{Win \%} & \textbf{Tie \%} & \textbf{Loss \%} \\ \midrule
\textbf{vs DPO} & 44.0 & 31.6 & 21.4 & 42.1 & 19.6 & 38.3 \\
\textbf{vs DPO-sft} & 45.4 & 34.4 & 20.2 & 45.2 & 18.9 & 35.9 \\
\textbf{vs RSO} & 44.2 & 35.5 & 20.3 & 45.1 & 17.6 & 37.3
\end{tabular}
}
\end{table}

%% file: icml2024/tables/self_normalized_baseline.tex
\begin{table}[h!] 
\centering
\begin{adjustbox}{max width=\columnwidth}
\begin{tabular}{lccc}
\toprule
\textbf{} & \textbf{BRAIn} & \textbf{ w/o self-norm} & \textbf{w/o baseline} \\
\midrule
\textbf{TL;DR} & 95.2 & 61.4 & 61.1 \\
\textbf{AnthropicHH} & 95.4 & 59.1 & 58.3 \\
\bottomrule
\end{tabular}
\end{adjustbox}
\caption{Effect of self-normalized baseline on the performance of various models.}
\label{tab:self_norm}
\end{table}

%% file: icml2024/tables/dpo_brain_gap.tex
\begin{table}[]
\centering
\caption{Win--rates (in \%age) obtained by incrementally augmenting \dposft; \dposft+IW: replace argmax by softmax; \dposft+IW+n: take softmax over $n=32$ outputs, instead of two outputs in 16 pairs}
\label{tab:bridging_gap}
\resizebox{\columnwidth}{!}{
\begin{tabular}{@{}lrrrr@{}}
\toprule
 & \textbf{DPO-sft} & \textbf{\begin{tabular}[c]{@{}r@{}}DPO-sft \\ +IW\end{tabular}} & \textbf{\begin{tabular}[c]{@{}r@{}}DPO-sft\\ +IW+n\end{tabular}} & \textbf{\shortname} \\ \midrule
\textbf{Train RM} & 87.37$\pm$0.91 & 89.21$\pm$0.85 & 93.30$\pm$0.69 & 95.40$\pm$0.57 \\
\textbf{LLM eval} & 66.78$\pm$1.29 & 69.28$\pm$1.27 & 73.89$\pm$1.21 & 74.36$\pm$1.17 \\ \bottomrule
\end{tabular}
}
\end{table}


%% file: icml2024/figures/num_outputs.tex
\begin{figure}[]
\centering
\begin{tikzpicture} \label{fig:num_outputs}
\begin{axis}[
    width=0.5\textwidth, 
    height=0.31\textwidth, 
    xlabel={Number of samples ($n$) per prompt},
    ylabel={Win-rate Against Gold (\%)},
    xmin=2, xmax=32,
    ymin=82, ymax=96,
    xmode=log, 
    log basis x={2}, 
    ytick={84, 86, 88, 90, 92, 94, 96},
    legend style={at={(0.65,0.7)}, anchor=west},
    ymajorgrids=true,
    grid style=dashed,
]

\addplot[
    color=blue,
    mark=square,
    ]
    coordinates {
    (2,89.21)(4,93.16)(8,95.53)(16,95.41)(32,95.42)
    };
\addplot[
    color=red,
    mark=square,
    ]
    coordinates {
    (2,83.57)(4,85.76)(8,86.43)(16,87.00)(32,88.00)
    };
    \legend{\shortname, \dposft}

\end{axis}
\end{tikzpicture}
\caption{Plot of Win-rate Against Gold as a function of the Number of Samples per Prompt.}
\label{fig:winrateplot}
\end{figure}
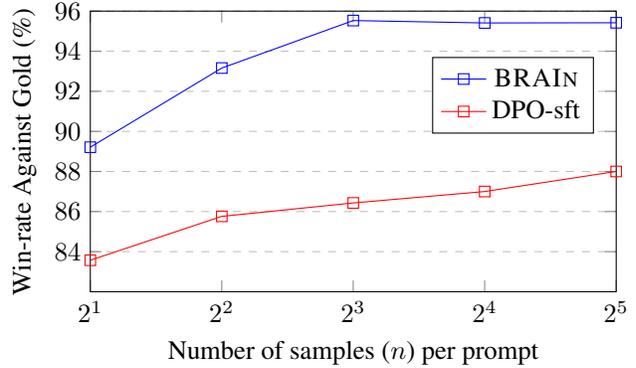

%% file: icml2024/sections/conclusions.tex
\section{Conclusion}


In this paper, we propose an LLM alignment algorithm called \shortname: \longname. The primary novelty in \shortname\ is the inclusion of a self-normalized baseline in the gradient estimate of distribution matching objectives for RL, that we refer to as \shortname\ gradient estimate.
Theoretically, the self-normalized baseline helps us to establish a connection between distribution matching methods and DPO, showing that DPO (DPO-sft, to be specific) is a special case of \shortname. 
We further establish that the \shortname\ gradient estimate is the gradient of a self-normalized version of KL divergence whose properties we intend to explore in future work. Additionally, we generalize the target distribution in PPO/DPO using Bayes' rule. The experimental results demonstrate the superiority of \shortname\ over other RLHF methods.

%% file: icml2024/sections/impact.tex
\section*{Impact Statement}
This paper presents work whose goal is to advance the field of Machine Learning. There are many potential societal consequences of our work, none which we feel must be specifically highlighted here.

%% file: icml2024/sections/appendix.tex
\section{Proof of Theorems}
\label{sec:proofs}

For the sake of completeness, we restate the definitions and the theorems here.

\selfkl*
\unbiasedestimator*

\begin{proof}
For unbiasedness of the \shortname\ gradient estimator, we need to show the following
\begin{equation}
        \multiexpectation \sum_{i=1}^n \left(\alphahaty{i} - \betahatthetay{i}\right) \nabla_\theta \log \policy{\outp_i} = -\nabla_\theta \KLnorm(\posterior{\outp} || \policy{\outp}) \,,
\end{equation}
where $\outpvec = (\outp_1, \ldots, \outp_n)$ is a sequence of $n$ outputs, the values of $\alpha$ and $\beta$ are as defined in \eqref{eq:alpha} and \eqref{eq:beta} respectively. The superscripts in $\alpha$ and $\beta$ denote their explicit dependence on the output sequence.

To prove the above, we expand each negative KL--divergence term in the self-normalized KL--divergence in terms of the entropy of normalized $\alphay{i}$ and the cross-entropy between self-normalized $\alphay{i}$ and $\betathetay{i}$.
\begin{equation}
    -\KLnorm = \multiexpectation \left[\entropy\left(\alphahaty{i}\right) + \sum_{i=1}^n \alphahaty{i} \log \betahatthetay{i} \right]
\end{equation}
Since the values $\alpha$ don't depend on $\theta$, they are discarded during the gradient computation. Hence, the gradient can be written as
\begin{align}
    -\nabla_\theta \KLnorm = \multiexpectation \sum_{i=1}^n \alphahaty{i} \nabla_\theta \left[\log \betahatthetay{i}\right]
\end{align}
Noting that $\betahatthetay{i} = \frac{\betathetay{i}}{\sum_{j=1}^n \betathetay{j}}$, we split the logarithm and compute the gradient of each term separately.
\begin{align}
    -\nabla_\theta \KLnorm &= \multiexpectation \sum_{i=1}^n \alphahaty{i}  \left[\nabla_\theta \log \betathetay{i} -\nabla_\theta \log \sum_{j=1}^n\betathetay{j}\right] \\
    &= \multiexpectation \sum_{i=1}^n \alphahaty{i}  \left[\nabla_\theta \log \betathetay{i} - \frac{\sum_{j=1}^n \nabla_\theta \betathetay{j}}{\sum_{j=1}^n \betathetay{j}}\right] \label{eq:log_split}
\end{align}
Next, we note that $\nabla_\theta \betathetay{i} = \betathetay{i} \nabla_\theta \log \betathetay{i}$ and
\begin{equation}
     \nabla_\theta \log \betathetay{i} = \nabla_\theta \left[\log \policy{\outp_i} - \log \proposal{\outp_i} \right] = \betathetay{i} \nabla_\theta \log \policy{\outp_i}
\end{equation}
Replacing these results back in equation \eqref{eq:log_split}, we get the desired result.
\begin{align}
    -\nabla_\theta \KLnorm &= \multiexpectation \sum_{i=1}^n \alphahaty{i}  \left[\nabla_\theta \log \policy{\outp_i} - \frac{\sum_{j=1}^n  \betathetay{j} \nabla_\theta \log \policy{\outp_j}}{\sum_{j=1}^n \betathetay{j}}\right] \\
    &= \multiexpectation \left[\frac{\sum_{i=1}^n \alphay{i} \nabla_\theta \log \policy{\outp_i}}  {\sum_{j=1}^n\alphay{j}}  - \frac{\sum_{i=1}^n  \betathetay{i} \nabla_\theta \log \policy{\outp_i}}{\sum_{j=1}^n \betathetay{j}}\right]  \\
    &= \multiexpectation \sum_{i=1}^n \left[\alphahaty{i} - \betahatthetay{i}\right] \nabla_\theta \log \policy{\outp_i}
\end{align}

\end{proof}

\sameminima*
\begin{proof}
    First, we will prove that $\KL = 0 \implies \KLnorm =0$ which is its minimum value. 
    We note that the self-normalized KL--divergence is the weighted sum of KL--divergence between the normalized values of $\alpha$ and $\beta$. Hence, from the property of KL--divergence, it can't be negative, that is, $\KLnorm \geq 0$.
    Now, lets assume that $\KL(\posterior{\outp}||\policy{\outp}) = 0$. By the property of KL--divergence, this implies that $\posterior{\outp} = \policy{\outp} \forall \outp \in \yset$. Thus:
    \begin{equation}
            \betathetay{i} = \frac{\policy{\outp_i}}{\proposal{\outp_i}} = \frac{\posterior{\outp_i}}{\proposal{\outp_i}} = \alphay{i}
    \end{equation}
    Incorporating it in the definition of self-normalized KL divergence in \eqref{eq:snkl}, we get:
    \begin{align} \label{eq:reverse_dir}
    \KLnorm & \left(\posterior{\outp} || \policy{\outp}\right) = \expectation_{\outppvec \sim \proposal{\outp}} \left[ \KL\left( \alphahaty{i} \left|\right| \alphahaty{i} \right) \right]  = 0
    \end{align}

    Instead of proving the other direction, we provide a constructive proof of its contrapositive, that is, 
    $\KL \neq 0 \implies \KLnorm \neq 0$. To see this, we note that $\KL \neq 0$ implies the existence of at least one output, say $\outp_+$, where the posterior and policy disagree. Without loss of generality, lets assume that $\posterior{\outp_+} > \policy{\outp_+}$. Since the probabilities must up to 1, there must exist at least one output, say $\outp_-$,  for which $\posterior{\outp_-} < \policy{\outp_-}$.

    Since self-normalized KL-divergence $\KLnorm$ has a KL-divergence $\KL$ term for every sequence of length $n$, we construct a sequence $\outpvec = (\outp_+, \outp_-, \ldots, \outp_-)$. 
    For such a sequence, all the values of $\alpha$ except the first one are equal. 
    We note that the importance weight of the first output, that is 
    $\alphay{+} = \frac{\posterior{\outp_+}}{\proposal{\outp_+}} > \frac{\policy{\outp_+}}{\proposal{\outp_+}} = \betathetay{+}$. 
    Similarly, the second importance weight $\alphay{-} = \frac{\posterior{\outp_-}}{\proposal{\outp_-}} < \frac{\policy{\outp_-}}{\proposal{\outp_-}} = \betathetay{-}$. Combining these two results we get $\frac{\alphay{-}}{\betathetay{-}} < 1 < \frac{\alphay{+}}{\betathetay{+}}$. This can further be written as $\frac{\alphay{-}}{\alphay{+}} < \frac{\betathetay{-}}{\betathetay{+}}$.
    
    Plugging in this result, the normalized values of $\alpha$ for the sequence are given by:
    \begin{equation}
        \frac{\alphay{1}}{\sum_{j=1}^n \alphay{j}} = \frac{\alphay{1}}{\alphay{+} + (n-1)\alphay{-}} = \frac{1}{1 + (n-1)\frac{\alphay{-}}{\alphay{+}}} > \frac{1}{1 + (n-1)\frac{\betathetay{-}}{\betathetay{+}}} = \frac{\betathetay{+}}{\betathetay{+} + (n-1)\betathetay{-}} = \frac{\betathetay{1}}{\sum_{j=1}^n \betathetay{j}}
    \end{equation}
    Thus the KL--divergence term for this particular sequence, that is $\KL\left( \alphahaty{i} || \betahatthetay{i}\right)$, is strictly greater than $0$. Since the support of proposal distribution includes the support of posterior and trainable policy, we get
    \begin{equation}
        \KLnorm(\posterior{\outp}||\policy{\outp}) = \expectation_{\outppvec \sim \proposal{\outp}} \left[ \KL\left( \alphahaty{i} || \betahatthetay{i} \right) \right]  > 0
    \end{equation}    
Together with equation~\eqref{eq:reverse_dir}, this proves that $\KL = 0 \implies \KLnorm =0$
\end{proof}

\bradleyterry*
\begin{proof}
The proof follows from the application of Bayes' rule and the parameterization of Bradley-Terry model given in \eqref{eq:btm_equate}.
\begin{align}\label{eq:alpha_base}
    \alphahaty{i} &= \frac{\posterior{\outp_i}/\proposal{\outp_i}}{\sum_{j=1}^n \posterior{\outp_j}/\proposal{\outp_j}}\\
    &= \frac{\frac{\prior{\outp_i}}{\proposal{\outp_i}}\times \frac{\rewardprob{\outp_i}}{\rewardprobx}}{\sum_{j=1}^n \frac{\prior{\outp_j}}{\proposal{\outp_j}}\times \frac{\rewardprob{\outp_j}}{\rewardprobx}} \label{eq:apply_bayes}\\
    &= \frac{\frac{\prior{\outp_i}}{\proposal{\outp_i}}\times \rewardprob{\outp_i}}{\sum_{j=1}^n \frac{\prior{\outp_j}}{\proposal{\outp_j}}\times \rewardprob{\outp_j}}\\
    &=  \frac{\frac{\prior{\outp_i}}{\proposal{\outp_i}}}{\sum_{j=1}^n \frac{\prior{\outp_j}}{\proposal{\outp_j}}\times \frac{\rewardprob{\outp_j}}{\rewardprob{\outp_i}}} \\
    &= \frac{\frac{\prior{\outp_i}}{\proposal{\outp_i}}}{\sum_{j=1}^n \frac{\prior{\outp_j}}{\proposal{\outp_j}}\times \exp{\left(\frac{ \rewardfn{\outp_j} - \rewardfn{\outp_i} }{\gamma}\right)}}\label{eq:apply_btm}\\
    &=  \frac{\exp(\frac{\rewardfn{\outp_i}}{\gamma} + \log \prior{\outp_i} - \log \proposal{\outp_i})}{\sum_{j=1}^n \exp( \frac{\rewardfn{\outp_j}}{\gamma} + \log \prior{\outp_j} - \log \proposal{\outp_j})} 
\end{align}
Here \eqref{eq:apply_bayes} follows by applying Bayes rule (see \eqref{eq:posterior}) while \eqref{eq:apply_btm} follows from the Bradley-Terry formulation in \eqref{eq:btm_equate}.
If we set $\proposal{\outp} = \prior{\outp}$, we get a softmax over the rewards as desired.
\end{proof}

\optimalpolicy*
\begin{proof}
We start with the definition of posterior in \cref{eq:posterior}, and use Bradley-Terry assumption to replace $\rewardprob{y}$ with a function of $\rewardfn{y}$.

In RHS of \cref{eq:posterior},  use total probability theorem to replace $\rewardprobx$ with 
$\sum_{\bar{\outp} \in \yset}\prior{\bar{\outp}}\rewardprob{\bar{\outp}}$ to get:
\begin{equation*}
    \posterior{\outp} =  \frac{\prior{\outp} \rewardprob{\outp}}{\sum_{\bar{\outp} \in \yset}\prior{\bar{\outp}}\rewardprob{\bar{\outp}}}
\end{equation*}
Next, move $\prior{\outp} \rewardprob{\outp}$ from the numerator to denominator:
\begin{equation*}
    \posterior{\outp} =  \frac{\prior{\outp}}{\left(\sum_{\bar{\outp} \in \yset}\prior{\bar{\outp}}\frac{\rewardprob{\bar{\outp}}}{ \rewardprob{\outp}}\right)}
\end{equation*}

Now, use \cref{eq:btm_equate} to replace $\frac{\rewardprob{\bar{\outp}}}{\rewardprob{\outp}}$ in the denominator with 
$\exp\left( \frac{\rewardfn{\bar{\outp}} - \rewardfn{\outp}}{\gamma} \right)$:
\begin{equation*}
    \posterior{\outp} =  \frac{\prior{\outp}}{\left(\sum_{\bar{\outp} \in \yset}\prior{\bar{\outp}}\exp\left( \frac{\rewardfn{\bar{\outp}} - \rewardfn{\outp}}{\gamma} \right)\right)}
\end{equation*}
Moving the common term $\exp (-\frac{\rewardfn{y}}{\gamma})$ from the denominator to numerator, we get:
\begin{equation*}
    \posterior{\outp} =  \frac{\prior{\outp}\exp(\frac{\rewardfn{\outp}}{\gamma})}{\left(\sum_{\bar{\outp} \in \yset}\prior{\bar{\outp}}\exp\left( \frac{\rewardfn{\bar{\outp}}}{\gamma} \right)\right)}
\end{equation*}

\end{proof}

\section{Details of Experimental Setup}

\subsection{Base models and datasets}
\label{sec:modeldatasetlinks}
In this section we provide links to all the publically available datasets and models used in our work.

We conduct experiments on two natural language generation tasks, \viz, summarization, and multi-turn dialog.
In the summarization task, we align CarperAI's summarization LM\footnote{ \href{https://huggingface.co/CarperAI/openai_summarize_tldr_sft}{CarperAI/openai\_summarize\_tldr\_sft}} to a Bradley-Terry reward model\footnote{\href{https://huggingface.co/CarperAI/openai_summarize_tldr_rm_checkpoint}{CarperAI/openai\_summarize\_tldr\_rm\_checkpoint}}. The base/SFT model has been trained on post-summary pairs from the train set of Reddit TL;DR dataset \footnote{\href{https://huggingface.co/datasets/CarperAI/openai_summarize_tldr}{datasets/CarperAI/openai\_summarize\_tldr}}. The reward model used for summarization has been trained on human preferences collected by \cite{stiennon2020learning} on outputs generated from a different SFT model where the prompts come from the TL;DR dataset.

In the multi-turn dialog task, we ensure that a open-ended chatbot's responses are \textbf{H}elpful and \textbf{H}armless. We use the Anthropic Helpful \& Harmless dataset~\cite{bai2022training} for RLHF training.
We use QLoRA-tuned Llama-7b\footnote{\href{https://huggingface.co/timdettmers/qlora-hh-rlhf-7b}{timdettmers/qlora-hh-rlhf-7b}} \cite{dettmers2023qlora} as SFT model. It has been trained on the preferred/chosen responses of Anthropic HH dataset.
A fine-tuned GPT-J \footnote{\href{https://huggingface.co/Dahoas/gptj-rm-static}{Dahoas/gptj-rm-static}} 
\cite{gpt-j} trained on a subset of the full Anthropic HH dataset is used as the reward model.
We train and evaluate using a subset\footnote{\href{https://huggingface.co/datasets/Dahoas/rm-static}{datasets/Dahoas/rm-static}} of Antrophic HH \cite{bai2022training} dataset.

\subsection{Other training details}\label{sec:other_train_details}
All the models are trained using the PEFT\footnote{https://huggingface.co/docs/peft/en/index} and Transformers\footnote{https://huggingface.co/docs/transformers/index} library. For Anthropic HH, we use QLoRA~\cite{dettmers2023qlora} for training \shortname and the other baselines. In particular, we use the same QLoRA hyperparameters (rank=$64$, $\alpha=16$) as used for supervised finetuning Llama-7B on Anthropic-HH dataset\footnote{https://huggingface.co/timdettmers/qlora-hh-rlhf-7b}. We use the Adam optimizer\cite{kingma2014adam} with a learning rate of $1e-5$, weight decay of $0.1$ and $\beta_1$ and $\beta_2$ set to $0.9$ and $0.95$ respectively.

For summarization, we use LoRA\cite{hu2021lora} (rank=$8$, $\alpha=32$). The optimizer, learning rate, weight decay and and $\beta$ values are the same as for Anthropic HH.

\subsection{Prompt for LLM-as-judge}
\label{sec:prompt-for-llm-as-judge}
In this section, we describe the prompts provided to the language models Mixtral-8x7B and GPT-4 for the purpose of comparing the gold standard outputs with the outputs generated by these models. Due to budget constraints, we evaluate only $500$ test prompts using GPT-4. The tasks involve acting as an impartial judge in evaluating responses or summaries provided by AI assistants.

\subsubsection{Prompt for Anthropic HH}
The prompt used for the Anthropic HH task is as follows:

"Please act as an impartial judge and evaluate the quality of the responses 
provided by the two AI assistants to the conversation displayed below. Your 
evaluation should consider correctness and helpfulness. You will be given a 
user conversation, assistant A's answer, and assistant B's answer. Your job 
is to evaluate which assistant's answer is better based on the user 
conversation so far. Begin your evaluation by comparing both assistants' 
answers with the user conversation so far. Identify and correct any mistakes. 
Avoid any position biases and ensure that the order in which the responses 
were presented does not influence your decision. Do not allow the length of 
the responses to influence your evaluation. Do not favor certain names of the 
assistants. Be as objective as possible. You should only evaluate the last 
utterance by both the assistants and not the full conversation. After 
providing your explanation, output your final verdict by strictly following 
this format: \"[[A]]\" if assistant A is better, \"[[B]]\" if assistant B is 
better, and \"[[C]]\" for a tie.\\
--------------------------------------------------\\
\{\{Conversation\}\}\\
--------------------------------------------------\\
Assistant B\\
 \{\{AssistantB\}\}\\
--------------------------------------------------\\
Assistant A\\
 \{\{AssistantA\}\}\\
--------------------------------------------------"

\subsubsection{Prompt for Summarization}
The prompt used for the Summarization task is outlined below:

"Please act as an impartial judge and evaluate the quality of the tldrs or 
summaries provided by the two AI assistants to the reddit post displayed 
below. Begin your evaluation by comparing both assistants' summaries with the 
reddit post so far. Do not allow the length of the summaries to influence your 
evaluation. After providing your explanation, output your final verdict by 
strictly following this format: \"[[A]]\" if assistant A is better, \"[[B]]\" 
if assistant B is better, and \"[[C]]\" for a tie.\\
--------------------------------------------------\\
Reddit Post\\
\{\{Conversation\}\} \\
-------------------------------------------------- \\
Assistant B \\
 \{\{AssistantB\}\}\\ 
-------------------------------------------------- \\
Assistant A \\
 \{\{AssistantA\}\}\\ 
--------------------------------------------------"